\documentclass[sn-standardnature]{sn-jnl}

\usepackage{stmaryrd}
\usepackage[all]{xy}

\renewcommand{\L}{\mathcal{L}}
\newcommand{\M}{\mathcal{M}}
\newcommand{\N}{\mathcal{N}}
\renewcommand{\P}{\mathcal{P}}

\newcommand{\pl}{\text{\bf P}}
\newcommand{\Ord}{\text{\bf Ord}}
\newcommand{\Card}{\text{\bf Card}}
\newcommand{\InfCard}{\text{\bf InfCard}}
\newcommand{\STML}{\text{\tt STML}}
\newcommand{\DTML}{\text{\tt DTML}}

\jyear{2021}%

\theoremstyle{thmstyleone}%
\newtheorem{theorem}{Theorem}%
\newtheorem{proposition}[theorem]{Proposition}%
\newtheorem{lemma}[theorem]{Lemma}%
\newtheorem{corollary}[theorem]{Corollary}%

\theoremstyle{thmstyletwo}%
\newtheorem{example}{Example}%
\newtheorem{remark}{Remark}%

\theoremstyle{thmstylethree}%
\newtheorem{definition}{Definition}%
\newtheorem{assumption}{Assumption}%

\begin{document}

\title[Transfinite Modal Logic for Bayesian Reasoning]{Transfinite Modal Logic: a Semi-quantitative Explanation for Bayesian Reasoning}


\author*[1]{\fnm{Xinyu} \sur{Wang}}\email{s2010404@jaist.ac.jp}

\affil*[1]{\orgdiv{School of Computer Science}, \orgname{Japan Advanced Institute of Science and Technology}, \orgaddress{\street{Asahidai 1--1}, \city{Nomi City}, \postcode{923--1211}, \state{Ishikawa Prefecture}, \country{Japan}}}

\abstract{Bayesian reasoning plays a significant role both in human rationality and in machine learning. In this paper, we introduce transfinite modal logic, which combines modal logic with ordinal arithmetic, in order to formalize Bayesian reasoning semi-quantitatively. Technically, we first investigate some nontrivial properties of ordinal arithmetic, which then enable us to expand normal modal logic's semantics naturally and elegantly onto the novel transfinite modal logic, while still keeping the ordinary definition of Kripke models totally intact. Despite all the transfinite mathematical definition, we argue that in practice, this logic can actually fit into a completely finite interpretation as well. We suggest that transfinite modal logic captures the essence of Bayesian reasoning in a rather clear and simple form, in particular, it provides a perfect explanation for Sherlock Holmes' famous saying, ``When you have eliminated the impossible, whatever remains, however improbable, must be the truth.'' We also prove a counterpart of finite model property theorem for our logic.}


\maketitle



\newpage

\section{Introduction}

Bayes' theorem, in spite of its seemingly simple mathematical form, not only is undoubtedly one of the most essential foundations of statistics, but also possesses extremely profound influence onto every corner of modern science~\cite{Freedman07}. The rule itself establishes one of the most popular philosophical interpretations of probability, namely Bayesian probability, together with the associative approach of Bayesian statistics~\cite{deFinetti17}; it has been successfully applied so as to construct a universal science theory, the Bayesian philosophy of science~\cite{Sprenger19}; and the recent boost of machine learning in AI indispensably relies on Bayes' rule as both theoretical basis and practical guidance~\cite{Bishop06}.

Logic, on the other hand, classically deals with prescribed truth values without any room for randomness~\cite{vanDalen04}. It may seem totally irrelevant to probability. Nonetheless, if we admit that both probability and logic --- or even only either one out of them --- is a universal feature of human being's intelligence, then these two flows of theories must intersect at some point~\cite{Audi98}. Admittedly, nowadays various types of probabilistic logics already exist~\cite{Nilsson86,Gilio05,Ognjanovic08,Nguyen19}, which \textit{extrinsically} merge the notion of probability into logical systems, for instance, by extending the possible space of a formula's truth value from only \textit{True} and \textit{False} to ranging upon all the real numbers between $0$ and $1$~\cite{Nilsson93,Biacino02,Dautovic21}. These kinds of logics could indisputably become useful for their own purposes in practice, though from a metaphysical perspective, they do not seem adequate for drawing a solid conclusion upon fundamental relation between logic and probability. Ontologically, if logic and probability are not two absolutely separated noumena, but instead constitutionally connected with each other, then we should reasonably expect that one notion can \textit{intrinsically} arise from the other, without the need to \textit{externally} introduce any new concept~\cite{Fitting04}.

After all, logic and probability are merely relevant to each other, but not inclusive by any means, thus generally speaking, we cannot precisely capture one concept entirely within the other~\cite{Hodges97}. Indeed from the technical aspect, the ordinary Kolmogorov probability axioms~\cite{Durrett19}, which comprehensively regulate every quantitative property of probability, do not seem very feasible to be fully deduced out solely based on some usual logical system. In this paper, however, we expand normal modal logic through combination with ordinal arithmetic and yield a novel kind of transfinite modal logic, whose interpretation provides a very interesting semi-quantitative explanation for Bayesian reasoning. Our model still keeps exactly the same as the traditional Kripke model, hence nothing new gets added substantially and thus we actually succeed in characterizing a fragment of the notion of probability \textit{internally} from modal logic. Quite roughly speaking, probability may also be viewed as a sort of modality, which essentially embeds itself inside normal modal logic~\cite{Pleasants09}: most commonly, $\Box\phi$ reads as ``$\phi$ is necessary'', but strictly speaking nothing is $100\%$ necessary and so the reading may sound closer to ``$\phi$ is \textit{almost} necessary with possibility near to $1$''~\cite{Fan14}. We shall further elaborate on our whole intuition in the rest of this paper.

Technically speaking, our transfinite modal logic depends on a few nontrivial results about ordinal arithmetic, which we will detailedly develop and prove in Section~\ref{sec.math} as mathematical preliminaries before we introduce the whole definition in Section~\ref{sec.dtml}. Later on we will also prove a counterpart of finite model property theorem in Section~\ref{sec.fmp}, which further consummates and justifies our logic. Although as its name suggests, transfinite modal logic involves infinite ordinals, we argue that practically it can also be interpreted in a finite manner. In Section~\ref{sec.det} we will demonstrate how transfinite modal logic is able to naturally explain Bayesian reasoning through a vivid example of Sherlock Holmes.

The rest of this paper is organized as follows: Section~\ref{sec.not} stipulates relevant notational conventions throughout this paper; Section~\ref{sec.stml} introduces simple transfinite modal logic as a sublanguage and then applies it to interpret probability; Section~\ref{sec.dif} analyzes in detail how we should expand the modal language so as to express dynamic Bayesian reasoning; Section~\ref{sec.math} develops a few nontrivial mathematical results about ordinal arithmetic, as a preparation for defining duplex transfinite modal logic in the following Section~\ref{sec.dtml}; Section~\ref{sec.det} demonstrates how Bayesian reasoning gets semi-quantitatively formalized by transfinite modal logic with a well known example about Sherlock Holmes; Section~\ref{sec.fmp} proves a counterpart of finite model property theorem for the logic, which also solidly guarantees a coherent theoretical background; Section~\ref{sec.con} at last briefly concludes the entire paper. Besides rigorous mathematical proofs for those pivotal theorems, a handful of vivid examples as well as careful remarks also scatter throughout this paper. Just as a piece of friendly advice, we think examples and remarks are extremely helpful for explaining our intuition behind the logic as clearly as possible, therefore, particular attention on those texts is recommended for the sake of a smooth understanding.

\section{Notational Conventions}\label{sec.not}

In this paper, we assume ZFC set theory~\cite{Jech03} and adopt the following notations:
\begin{itemize}
    \item $\subseteq$ is for subset (subclass) and $\subset$ is for proper subset (subclass).
    \item $\Ord$ is the class of ordinals.
    \item $\Card$ is the class of cardinals. $\Card\subset\Ord$.
    \item $\InfCard\subset\Card$ is the class of infinite cardinals. $\InfCard$ is isomorphic to $\Ord$, therefore, each $\aleph_\alpha\in\InfCard$ is uniquely indexed by each $\alpha\in\Ord$.
    \item Lowercase Greek letters $\alpha,\beta,\gamma,\zeta,\eta,\theta,\rho,\sigma,\tau$ are used for representing ordinals; $\kappa,\lambda$ for representing cardinals; $\phi,\psi$ for representing formulae.
    \item All kinds of arithmetic, including addition ($+$), multiplication ($\cdot$) and exponentiation, are ordinal arithmetic. (i.e., they should never be understood as cardinal arithmetic, even if both operands are cardinals.)
\end{itemize}

\section{Simple Transfinite Modal Logic}\label{sec.stml}

To start with, we define simple transfinite modal logic in this section. Syntactically, simple transfinite modal logic bears exactly the same language as normal modal logic, while semantically its only modification is that, when evaluating formula $\Box\phi$ on a Kripke model, cardinality gets taken into consideration just by the most natural way that readers could possibly think of.
\begin{definition}[Language $\STML$]
Let $\pl$ be a non-empty set of propositions. Language $\STML$ is recursively defined by the following Backus-Naur form, where $p\in\pl$:
\begin{align*}
    \phi::=p\mid\neg\phi\mid(\phi\land\phi)\mid\Box\phi
\end{align*}
\end{definition}
\begin{definition}[Kripke Model]\label{def.mod}
A Kripke model $\M$ is a triple $(S,R,V)$ where:
\begin{itemize}
    \item $S$ is a non-empty set of possible worlds.
    \item $R\subseteq S\times S$ is a binary relation over $S$.
    \item $V:\pl\to\P(S)$ is the valuation function.
\end{itemize}
A pointed (Kripke) model $\M,s$ is a Kripke model $\M$ with a possible world $s\in S$.
\end{definition}
\begin{definition}[Successor Cardinality of Pointed Model]\label{def.scar}
For any pointed model $\M,s$, we define its successor cardinality as $\kappa_s=\mid\{t\in S:sRt\}\mid$.
\end{definition}
\begin{remark}\label{rem.scar}
Intuitively, ``successor cardinality of pointed model $\M,s$'' in Definition~\ref{def.scar} is just the cardinality of successors of $s$.
\end{remark}
\begin{definition}[Semantics of $\STML$]\label{def.ssem}
For any pointed model $\M,s$, semantics of language $\STML$ is recursively defined as the following:
\begin{align*}
    \M,s\vDash p\iff & s\in V(p)\\
    \M,s\vDash\neg\phi\iff & \text{not }\M,s\vDash\phi\\
    \M,s\vDash\phi\land\psi\iff & \M,s\vDash\phi\text{ and }\M,s\vDash\psi\\
    \M,s\vDash\Box\phi\iff & \left\{\begin{array}{llll}
        \mid\{t\in S:sRt\land(\M,t\vDash\phi)\}\mid> &\\
        \mid\{t\in S:sRt\land\neg(\M,t\vDash\phi)\}\mid, & & & \text{if }\kappa_s\in\InfCard\\
        \\
        \forall t\in S\text{ such that }sRt,\M,t\vDash\phi, & & & \text{otherwise}
    \end{array}\right.
\end{align*}
\end{definition}
Readers should easily reason out that intuitively, semantics of $\Box\phi$ in Definition~\ref{def.ssem} just says that when $s$ has infinite many successors, $\Box\phi$ holds on $s$ if and only if successors on which $\phi$ holds are more than successors on which $\phi$ does not hold, namely, $\phi$ holds on \textit{almost} all the successors and so is \textit{almost} necessary. The following Example~\ref{ex.sim} concisely demonstrates such a case.
\begin{example}\label{ex.sim}
In this Kripke model $\M_{\ref{ex.sim}}$, from $s$ there is only one successor on which $\neg p$ holds but $\aleph_0$ many successors on which $p$ holds. Therefore we have $\M_{\ref{ex.sim}},s\vDash\Box p$, simply ignoring the existence of that single $\neg p$ successor as it is \textit{almost} impossible.
$$
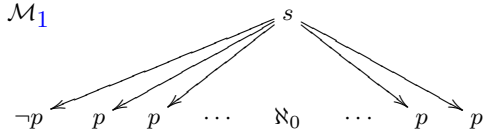
\xymatrix@R=25pt@C=10pt{
\M_{\ref{ex.sim}} & & & & s\ar@{->}[dllll]\ar@{->}[dlll]\ar@{->}[dll]\ar@{->}[drr]\ar@{->}[drrr] & & &\\
\neg p & p & p & \cdots & \aleph_0 & \cdots & p & p
}$$
\end{example}
\begin{remark}\label{rem.fin}
A finitist might argue that in reality, every physical object is finite and thus such a model in Example~\ref{ex.sim} should not be acceptable~\cite{Ye11}. However, this paper does not intend to survey on metaphysical problems in mathematical philosophy, and here we make use of infinite cardinality simply as a convenient tool for theoretical beauty; in fact as for practical usage, we think it is perfectly okay to interpret $\aleph_0$ as certain sufficiently large natural number $N\in\omega$, such that any common natural number, for example $5$, $10$ or even $100$, is much smaller than $N$. Of course, it depends on the specific situation in application to determine the actual upper range of a ``common'' natural number, as well as how large $N$ must be in order to be counted as ``sufficiently'' large. Furthermore, if $\aleph_1$ also appears in the Kripke model, then it can be interpreted as an even larger natural number $N'\in\omega$, such that the result of any common elementary arithmetic on $N$, for example $2N(=N+N)$ or $N^2(=N\cdot N)$, is still much smaller than $N'$.

As an analogy, readers are also suggested to recall the big O notation, which is now commonly used for denoting the computational complexity of an algorithm~\cite{Sipser06}. Although in principal, the input scale of any computer algorithm must be finite during a single run, the big O notation well captures the algorithm's approximate behaviour as long as the input scale is large enough so as to be reasonably considered as \textit{almost} infinite. Our idea here simply resembles such an intuition. In short, we would recommend just freely choosing to understand $\aleph_0$ in this paper at the ontological aspect either as actual infinity or as certain sufficiently large natural number, whichever could make readers more comfortable.
\end{remark}

\section{Toward Modal Logic for Bayesian Reasoning}\label{sec.dif}

Now that we have defined simple transfinite modal logic in Section~\ref{sec.stml}, where $\Box\phi$ holds on a possible world $s$ with infinite many successors if and only if $\phi$ holds on \textit{almost} all the successors of $s$, in other words, the set of successors of $s$ on which $\phi$ does not hold must be smaller in cardinality and thus comparatively negligible. Literally, this intuition is so simple. Nevertheless, our final goal is to explain Bayesian reasoning with modal logic, indicating that we have to deal with something like dynamic belief revision~\cite{Baltag08,vanBenthem09}, for instance, $\Box\phi$ holds a priori and hence $\phi$ is \textit{almost} necessary as the prior probability distribution, but then upon obtaining new evidence, the posterior probability distribution gets calculated according to Bayes' rule and so $\Box\phi$ may not hold any more. To provide a more specific context, suppose you have bought a lottery and $\phi$ stands for that you do not win the first prize, then a priori assertion of $\Box\phi$ could be reasonably drawn, but if you later check the result and (very luckily!) find out that you really do win the first prize, then the default $\Box\phi$ should now be rationally abandoned by Bayesian reasoning.

Hence here follows the natural question: how should we model the dynamics of Bayesian reasoning? Readers might quickly think of public announcement logic, where update of information is achieved through deleting unwanted possible worlds form the original Kripke model~\cite{vanDitmarsch08}. Nevertheless, such kind of approach now faces a serious endogenous difficulty: by deleting possible worlds, the Kripke model inevitably gets smaller and smaller, but more complicated new structures can never emerge. This feature may not present to be a big problem for public announcement logic, which is intended to deal with knowledge about the ultimately certain truth, however, it does not seem to fit in well with Bayesian belief that could arbitrarily invert for (potentially) infinite many times.

Therefore, we perceive that another natural approach is more plausible: just like temporal logic for discrete time~\cite{Goldblatt92}, the Kripke model is generally stratified as a tree, and exploration over the next layer corresponds to receiving new information and then adjusting the posterior probability distribution by Bayesian reasoning. Intuitively, such modelling indeed sounds quite similar to human being's natural reasoning process. The following Example~\ref{ex.fin} demonstrates such intuition with a very simple finite model:
\begin{example}\label{ex.fin}
In this Kripke model $\M_{\ref{ex.fin}}$, by default namely among children of $s$, there exist two possible worlds $t_1$ and $t_2$ on which $p$ holds but only one possible world $t_0$ on which $\neg p$ holds, hence representing that the prior probability distribution is $\Pr[p]=\frac{2}{3}$ and $\Pr[\neg p]=\frac{1}{3}$. However among grandchildren of $s$, namely after one round of Bayesian update, we find out that there exist four evidences to support $t_0$ but for either $t_1$ or $t_2$ there is only one evidence, therefore, the posterior probability distribution should be $\Pr_\text{post}[p]=\frac{2}{6}=\frac{1}{3}$ and $\Pr_\text{post}[\neg p]=\frac{4}{6}=\frac{2}{3}$. The intuitive explanation being such in natural language, however as a clarification, we do not expect our transfinite modal logic expressive enough to be able to calculate out specific numerical values of probability such as $\frac{1}{3}$ or $\frac{2}{3}$.
$$
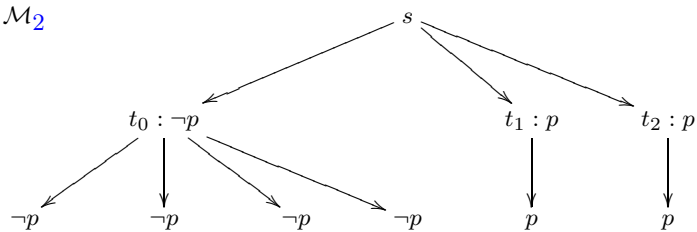
\xymatrix@R=25pt@C=25pt{
\M_{\ref{ex.fin}} & & & s\ar@{->}[dll]\ar@{->}[dr]\ar@{->}[drr] & &\\
& t_0:\neg p\ar@{->}[dl]\ar@{->}[d]\ar@{->}[dr]\ar@{->}[drr] & & & t_1:p\ar@{->}[d] & t_2:p\ar@{->}[d]\\
\neg p & \neg p & \neg p & \neg p & p & p
}$$
\end{example}
Take Example~\ref{ex.fin} as a reference, we are tempted to analogize that for transfinite modal logic, since formula $\Box\phi$ expresses that $\phi$ is \textit{almost} necessary as of the prior probability distribution, that $\phi$ is \textit{almost} necessary as of the posterior probability distribution should then be expressed by formula $\Box\Box\phi$ or something alike. This overall direction does not seem too problematic, however, there remain a few essential technical details with complexity, which require our very careful examination. Firstly let us consider the following Example~\ref{ex.inf}:
\begin{example}\label{ex.inf}
In this Kripke model $\M_{\ref{ex.inf}}$, the prior probability distribution is that there are $\aleph_0$ many successors that satisfy $p$ while $\neg p$ only holds on a single successor $t_0$, hence $p$ is a priori \textit{almost} necessary, namely $\M_{\ref{ex.inf}},s\vDash\Box p$. However later on, we newly acquire as many as $\aleph_0$ subsequent evidences to support exactly the very branch of possibility $t_0$, while for any other successor of $s$ only one following evidence gets provided. Hence intuitively, by Bayesian reasoning we should infer that based on current information, the posterior probability of $\neg p$ has now become comparable to probability of $p$ and thus the possibility of $\neg p$ can no longer be neglected as \textit{almost} impossible.
$$
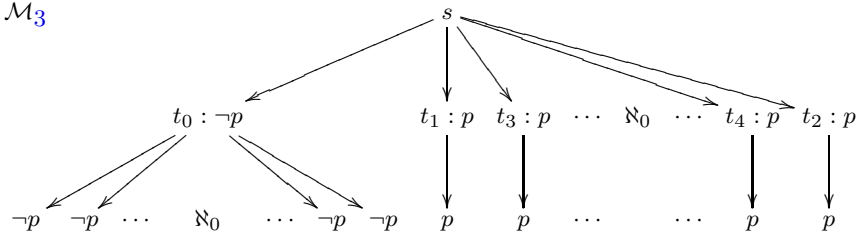
\xymatrix@R=25pt@C=2.5pt{
\M_{\ref{ex.inf}} & & & & & & & s\ar@{->}[dllll]\ar@{->}[d]\ar@{->}[dr]\ar@{->}[drrrrr]\ar@{->}[drrrrrr] & & &\\
& & & t_0:\neg p\ar@{->}[dlll]\ar@{->}[dll]\ar@{->}[drr]\ar@{->}[drrr] & & & & t_1:p\ar@{->}[d] & t_3:p\ar@{->}[d] & \cdots & \aleph_0 & \cdots & t_4:p\ar@{->}[d] & t_2:p\ar@{->}[d]\\
\neg p & \neg p & \cdots & \aleph_0 & \cdots & \neg p & \neg p & p & p & \cdots & & \cdots & p & p
}$$
Such an intuition sounds quite plausible, but then how may we formally express the above analysis by a modal logic formula? Notice that however, according to the semantics defined in Definition~\ref{def.ssem}, now we still have $\M_{\ref{ex.inf}},s\vDash\Box\Box p$; in fact from the viewpoint of $s$, $\Box\Box p$ differs from $\Box p$ in no way at all, as actually for any successor $t\in S$ such that $sRt$, $\M_{\ref{ex.inf}},t\vDash\Box p$ if and only if $\M_{\ref{ex.inf}},t\vDash p$. If we take a closer look into this issue, we may be able to recognize that roughly speaking, what we really want to mean is not $\Box\Box p$, which essentially means $\Box(\Box p)$ as its standard valuating order; instead, we want to mean $(\Box\Box)p$, that is to say, firstly view $(\Box\Box)$ as an integrated modality corresponding to the combined relation $(R\circ R)$. In fact, this exchange of valuating order does not make any difference for normal modal logic, but here it matters. For the sake of clarity, we shall denote this bundled modality as a single notation $\boxbox$ in the rest of this paper. We would also like to suggest denoting $\boxbox$ alternatively as $\Box^2$, a notation that can be naturally generalized onto $\Box^3$ and so on, which are capable of describing multiple steps of Bayesian reasoning. In this paper nevertheless, we prefer neatly focusing on the general characters of transfinite modal logic rather than diving into cumbersome details, so we shall restrict our following discussion in the rest of this paper only up to $\Box^2$, i.e., $\boxbox$. The technique for $\Box^3$ and others is simply similar in principle.
\end{example}
In a word, we conclude that in Example~\ref{ex.inf}, we should have $\M_{\ref{ex.inf}},s\vDash\neg\boxbox p$, although semantics of the new modality $\boxbox$ still remains undefined, which constitutes our major task at present. As suggested in Example~\ref{ex.inf}, $\boxbox$ is meant nearly as $(\Box\Box)$ for the combined relation $(R\circ R)$, whose semantics might then be defined in the same way as Definition~\ref{def.ssem}. Such definition approach, however, causes serious flaws in philosophy, as the following Example~\ref{ex.flaw} reveals:
\begin{example}\label{ex.flaw}
In this Kripke model $\M_{\ref{ex.flaw}}$, the root $s$ is connected to $\aleph_0$ many successors, each of which is just the same copy of the pointed model $\M_{\ref{ex.sim}},s$ in Example~\ref{ex.sim}. By symmetry, there is no reason for us to anticipate an outcome different from Example~\ref{ex.sim}, namely, we should expect that $\M_{\ref{ex.flaw}},s\vDash\boxbox p$. However, simply counting the cardinality of grandchildren of $s$ would lead to the opposite statement $\M_{\ref{ex.flaw}},s\vDash\neg\boxbox p$, since both grandchildren on which $p$ holds and grandchildren on which $\neg p$ holds are equally $\aleph_0$ in terms of cardinality. This conclusion is not acceptable, as apparently there are much ``more'' chances of $p$ than $\neg p$ just from harmless intuition. An alternative viewpoint results into a very strong objection as well: as discussed in Remark~\ref{rem.fin}, suppose in practice, $\aleph_0$ is actually interpreted as a synonym of certain natural number $N\in\omega$ such that $N$ is far greater than $1$, but then $N^2$ should also be far greater than $N$, thus the case of $p$ still bears much greater possibility than the case of $\neg p$ and so is indeed \textit{almost} necessary.
$$
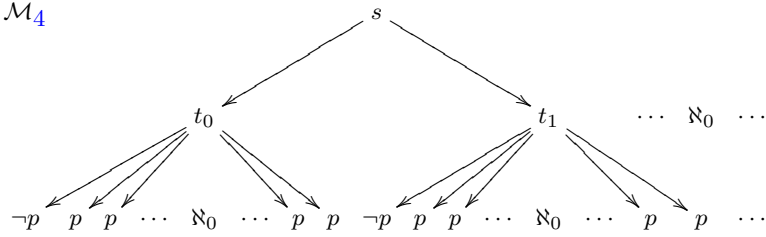
\xymatrix@R=25pt@C=2.5pt{
\M_{\ref{ex.flaw}} & & & & & & & & s\ar@{->}[dllll]\ar@{->}[drrrr] & & & & & & & &\\
& & & & t_0\ar@{->}[dllll]\ar@{->}[dlll]\ar@{->}[dll]\ar@{->}[drr]\ar@{->}[drrr] & & & & & & & & t_1\ar@{->}[dllll]\ar@{->}[dlll]\ar@{->}[dll]\ar@{->}[drr]\ar@{->}[drrr] & & \cdots & \aleph_0 & \cdots\\
\neg p & p & p & \cdots & \aleph_0 & \cdots & p & p & \neg p & p & p & \cdots & \aleph_0 & \cdots & p & p & \cdots
}$$
\end{example}
Example~\ref{ex.flaw} warns us that, essentially, flaws occur because cardinals are too coarse-grained, so that cardinal-arithmetic square of $\aleph_\alpha$ still equals $\aleph_\alpha$ for any $\aleph_\alpha\in\InfCard$~\cite{Kunen80}. Immediately, we get reminded of ordinals, which are finer-grained than cardinals since $\alpha^2\neq\alpha$ for any $\alpha\in\Ord$ such that $\alpha>1$. It seems that ordinals are good candidate of replacement for cardinals, but \textit{beware}. Unlike cardinals, which can always yield an absolute value for any fixed set, ordinals rely on well orders, which originally do not exist in a Kripke model and we have never planned to extra introduce from elsewhere, either. Such fact raises the major hurdle. Later on in Section~\ref{sec.dtml}, we shall rigorously define the semantics of modality $\boxbox$, and the accordingly expanded modal language is called duplex transfinite modal logic. Notably, the Kripke model still keeps its original basic form as in Definition~\ref{def.mod}. No extra well orders are introduced into the Kripke model, but we manage to resolve the flaw in Example~\ref{ex.flaw} with ingenious ordinal arithmetic. Before introducing this duplex transfinite modal logic, nonetheless, several crucial nontrivial results on ordinal arithmetic have to be firstly established in the following Section~\ref{sec.math}.

\section{Ordinal Arithmetic}\label{sec.math}

\subsection{Ordinal Logarithm}

\begin{lemma}\label{lm.car}
For any $\alpha,\beta\in\Ord$, if $0<\beta<\aleph_\alpha$, then $\beta+\aleph_\alpha=\beta\cdot\aleph_\alpha=\aleph_\alpha$.
\end{lemma}
\begin{proof}
Since $\beta>0$, namely $\beta\geqslant1$, obviously $\beta+\aleph_\alpha\geqslant\aleph_\alpha$ and $\beta\cdot\aleph_\alpha\geqslant\aleph_\alpha$.

To a contradiction suppose $\beta+\aleph_\alpha>\aleph_\alpha$, namely $\aleph_\alpha\in\beta+\aleph_\alpha$. As $\aleph_\alpha$ is an infinite cardinal, it is a limit ordinal, so $\beta+\aleph_\alpha=\bigcup\limits_{\gamma<\aleph_\alpha}(\beta+\gamma)$, and thus there exists $\gamma<\aleph_\alpha$ such that $\aleph_\alpha\in\beta+\gamma$, namely $\beta+\gamma>\aleph_\alpha$, hence $\mid\beta+\gamma\mid\geqslant\aleph_\alpha$. However, since $\aleph_\alpha$ is a cardinal and $\beta,\gamma<\aleph_\alpha$, $\mid\beta\mid,\mid\gamma\mid<\aleph_\alpha$, so $\mid\beta+\gamma\mid=\max\{\mid\beta\mid,\mid\gamma\mid\}<\aleph_\alpha$, a contradiction. Therefore $\beta+\aleph_\alpha=\aleph_\alpha$.

The case for multiplication is just similar.
\end{proof}
\begin{definition}[Relation $\sim_\alpha$]
Let $\alpha\in\Ord$ be fixed. For any $\beta,\gamma\in\Ord$ such that $\beta>0\lor\gamma>0$, define $\beta\sim_\alpha\gamma$ iff $\beta<\gamma\cdot\aleph_\alpha\land\gamma<\beta\cdot\aleph_\alpha$. Besides, let $0\sim_\alpha0$.
\end{definition}
\begin{proposition}
For any fixed $\alpha\in\Ord$, $\sim_\alpha$ is an equivalence relation over $\Ord$.
\end{proposition}
\begin{proof}
Both reflexivity and symmetry of relation $\sim_\alpha$ are quite straightforward to check. Cases concerning about $0$ are also rather trivial. So focusing on transitivity, take arbitrary three ordinals $\beta,\gamma,\zeta\neq0$ such that $\beta\sim_\alpha\gamma\land\gamma\sim_\alpha\zeta$. Thus $\beta<\gamma\cdot\aleph_\alpha$ and $\gamma<\zeta\cdot\aleph_\alpha$, so there exist $\eta,\theta<\aleph_\alpha$ such that $\beta<\gamma\cdot\eta$ and that $\gamma<\zeta\cdot\theta$, hence $\beta<\zeta\cdot\theta\cdot\eta<\zeta\cdot\aleph_\alpha$. Similarly $\zeta<\beta\cdot\aleph_\alpha$, therefore $\beta\sim_\alpha\zeta$.
\end{proof}
\begin{definition}[Equivalence Class ${[\beta]_\alpha}$ and Relation $<$]\label{def.equi}
Let $\alpha\in\Ord$ be fixed. For any $\beta\in\Ord$, we denote its corresponding equivalence class over $\Ord$ derived from relation $\sim_\alpha$ as $[\beta]_\alpha=\{\gamma\in\Ord:\gamma\sim_\alpha\beta\}$. We then naturally define relation $<$ over the class of equivalence classes $\Ord/\sim_\alpha$: for any $\beta,\gamma\in\Ord$, $[\beta]_\alpha<[\gamma]_\alpha$ iff $\neg(\beta\sim_\alpha\gamma)\land\beta<\gamma$.
\end{definition}
\begin{proposition}
Relation $<$ is well defined in Definition~\ref{def.equi}, and is also a well order.
\end{proposition}
\begin{proof}
Obviously, this relation $<$ over $\Ord/\sim_\alpha$ just reflects the original relation $<$ over $\Ord$.
\end{proof}
\begin{lemma}\label{lm.log}
For any $\alpha,\beta\in\Ord$ such that $\alpha\geqslant1$, $\beta>1$, there uniquely exists a group of three ordinals $\gamma,\zeta,\eta\in\Ord$ such that $\alpha=\beta^\gamma\cdot\zeta+\eta$, $1\leqslant\zeta<\beta$, $\eta<\beta^\gamma$.
\end{lemma}
\begin{proof}
For existence. Since $\beta>1$, power of $\beta$ has no upper bound. Hence let $\gamma$ be the minimal ordinal such that $\alpha<\beta^{\gamma+1}$, and as $\beta>1$, $\beta^\gamma\geqslant1$. Then by ordinal division, there uniquely exists a group of two ordinals $\zeta,\eta$ such that $\alpha=\beta^\gamma\cdot\zeta+\eta$, $\eta<\beta^\gamma$. Suppose $\zeta\geqslant\beta$, then $\alpha=\beta^\gamma\cdot\zeta+\eta\geqslant\beta^\gamma\cdot\zeta\geqslant\beta^\gamma\cdot\beta=\beta^{\gamma+1}$, a contradiction, so $\zeta<\beta$. Suppose $\zeta=0$, then $\alpha=\beta^\gamma\cdot\zeta+\eta=\eta<\beta^\gamma$, and as $\alpha\geqslant1$, $\gamma>0$. If $\gamma$ is a successor ordinal, namely there exists an ordinal $\theta$ such that $\gamma=\theta+1$, then $\alpha<\beta^{\theta+1}$ but $\theta<\gamma$, a contradiction. If $\gamma$ is a limit ordinal, then $\alpha\in\beta^\gamma=\bigcup\limits_{\rho<\gamma}\beta^\rho$, so there exists an ordinal $\theta<\gamma$ such that $\alpha\in\beta^\theta$, then $\alpha<\beta^\theta<\beta^{\theta+1}$, a contradiction. So $\zeta\neq0$, therefore $1\leqslant\zeta<\beta$.

For uniqueness. Suppose there exist two groups of three ordinals $\gamma_1,\zeta_1,\eta_1$ and $\gamma_2,\zeta_2,\eta_2$, such that $\alpha=\beta^{\gamma_1}\cdot\zeta_1+\eta_1=\beta^{\gamma_2}\cdot\zeta_2+\eta_2$, $1\leqslant\zeta_1,\zeta_2<\beta$, $\eta_1<\beta^{\gamma_1}$, $\eta_2<\beta^{\gamma_2}$. Hence $\beta^{\gamma_1}\leqslant\alpha=\beta^{\gamma_1}\cdot\zeta_1+\eta_1<\beta^{\gamma_1}\cdot\zeta_1+\beta^{\gamma_1}=\beta^{\gamma_1}\cdot(\zeta_1+1)\leqslant\beta^{\gamma_1}\cdot\beta=\beta^{\gamma_1+1}$, and similarly $\beta^{\gamma_2}\leqslant\alpha<\beta^{\gamma_2+1}$, so it is easy to see $\gamma_1=\gamma_2$. Then by ordinal division, $\zeta_1=\zeta_2$ and $\eta_1=\eta_2$.
\end{proof}
\begin{definition}[Ordinal Logarithm]
For any $\alpha,\beta\in\Ord$ such that $\alpha\geqslant1$, $\beta>1$, referring to Lemma~\ref{lm.log}, there unique exists $\gamma\in\Ord$, which we denote as $\gamma=\log_{\beta}\alpha$.
\end{definition}
\begin{theorem}\label{th.log}
Let $\alpha\in\Ord$ be fixed. For any $\beta,\gamma\in\Ord$ such that $\beta,\gamma\geqslant1$, $[\beta]_\alpha<[\gamma]_\alpha\iff\log_{\aleph_\alpha}\beta<\log_{\aleph_\alpha}\gamma$.
\end{theorem}
\begin{proof}
By Lemma~\ref{lm.log}, there uniquely exist two groups of three ordinals $\zeta,\eta,\theta$ and $\rho,\sigma,\tau$, such that $\beta=\aleph_\alpha^\zeta\cdot\eta+\theta$, $\gamma=\aleph_\alpha^\rho\cdot\sigma+\tau$, $1\leqslant\eta,\sigma<\aleph_\alpha$, $\theta<\aleph_\alpha^\zeta$, $\tau<\aleph_\alpha^\rho$. Therefore by also making use of Lemma~\ref{lm.car}:
\begin{align*}
    [\beta]_\alpha<[\gamma]_\alpha\iff & \beta\cdot\aleph_\alpha\leqslant\gamma\\
    \iff & (\aleph_\alpha^\zeta\cdot\eta+\theta)\cdot\aleph_\alpha\leqslant\aleph_\alpha^\rho\cdot\sigma+\tau\\
    \iff & \aleph_\alpha^{\zeta+1}\leqslant\aleph_\alpha^\rho\cdot\sigma+\tau\\
    \iff & \zeta+1\leqslant\rho\\
    \iff & \zeta<\rho\\
    \iff & \log_{\aleph_\alpha}\beta<\log_{\aleph_\alpha}\gamma
\end{align*}
\end{proof}

\subsection{Transfinite Ordinal Addition}

\begin{definition}[Transfinite Ordinal Addition]
Let $\alpha\in\Ord$ be fixed, and $f_\alpha:\alpha\to\Ord$ be a fixed function. For any $\beta\in\Ord$ such that $\beta\leqslant\alpha$, we define $\sum\limits_\beta f_\alpha$ by transfinite induction as the following:
\begin{itemize}
    \item If $\beta=0$, then let $\sum\limits_0 f_\alpha=0$.
    \item If $\beta$ is a successor ordinal, namely there exists an ordinal $\gamma$ such that $\beta=\gamma+1$, since $\beta\leqslant\alpha$, we have $\gamma<\alpha$, namely $\gamma\in\alpha$, so let $\sum\limits_\beta f_\alpha=(\sum\limits_\gamma f_\alpha)+f_\alpha(\gamma)$.
    \item If $\beta$ is a limit ordinal, then let $\sum\limits_\beta f_\alpha=\bigcup\limits_{\gamma<\beta}(\sum\limits_\gamma f_\alpha)$.
\end{itemize}
Finally, we denote $\sum\limits_\alpha f_\alpha$ simply as $\sum f_\alpha$.
\end{definition}
\begin{definition}[Relation $<$]\label{def.add}
Let $\alpha\in\Ord$ be fixed, and $f_\alpha:\alpha\to\Ord$ be a fixed function. Define relation $<$ over the set $\{(\beta,\gamma):\beta\in\alpha,\gamma\in f_\alpha(\beta)\}$ as the following: for any $(\beta_1,\gamma_1),(\beta_2,\gamma_2)\in\{(\beta,\gamma):\beta\in\alpha,\gamma\in f_\alpha(\beta)\}$, if $\beta_1\neq\beta_2$, then $(\beta_1,\gamma_1)<(\beta_2,\gamma_2)$ iff $\beta_1<\beta_2$; otherwise, $(\beta_1,\gamma_1)<(\beta_2,\gamma_2)$ iff $\gamma_1<\gamma_2$.
\end{definition}
\begin{lemma}\label{lm.add}
Relation $<$ defined in Definition~\ref{def.add} is a well order isomorphic to $\sum f_\alpha$.
\end{lemma}
\begin{proof}
It is not difficult to see that relation $<$ over the set $\{(\beta,\gamma):\beta\in\alpha,\gamma\in f_\alpha(\beta)\}$ is a total order. To show that it is also well founded, for any non-empty subset $A\subseteq\{(\beta,\gamma):\beta\in\alpha,\gamma\in f_\alpha(\beta)\}$ such that $A\neq\emptyset$, let $A_\beta=\{\beta\in\alpha:\exists\gamma\in f_\alpha(\beta),(\beta,\gamma)\in A\}$. Obviously $A_\beta\subseteq\alpha$ and $A_\beta\neq\emptyset$, since ordinals are well ordered, let $\beta_0\in A_\beta$ be the minimal element in $A_\beta$, and let $A_\gamma=\{\gamma\in f_\alpha(\beta_0):(\beta_0,\gamma)\in A\}$. Obviously $A_\gamma\subseteq f_\alpha(\beta_0)$ and $A_\gamma\neq\emptyset$, since ordinals are well ordered, let $\gamma_0\in A_\gamma$ be the minimal element in $A_\gamma$. Then $(\beta_0,\gamma_0)\in A$, and it is easy to reason that $(\beta_0,\gamma_0)$ is the minimal element in $A$. Hence relation $<$ over the set $\{(\beta,\gamma):\beta\in\alpha,\gamma\in f_\alpha(\beta)\}$ is well founded, so it is a well order.

In the following we show that the ordinal of well order relation $<$ over the set $\{(\beta,\gamma):\beta\in\alpha,\gamma\in f_\alpha(\beta)\}$ is $\sum f_\alpha$, by transfinite induction on $\alpha$:
\begin{itemize}
    \item If $\alpha=0$, namely $\alpha=\emptyset$ so obviously the set $\{(\beta,\gamma):\beta\in\alpha,\gamma\in f_\alpha(\beta)\}=\emptyset$ and $\sum f_\alpha=0$.
    \item If $\alpha$ is a successor ordinal, namely there exists an ordinal $\zeta$ such that $\alpha=\zeta+1$, let $g_\zeta:\zeta\to\Ord=\{(\beta,f_\alpha(\beta)):\beta\in\zeta\}\subset f_\alpha$ be another function, and by induction hypothesis we know that the ordinal of well order relation $<$ over the set $\{(\beta,\gamma):\beta\in\zeta,\gamma\in g_\zeta(\beta)\}$ is $\sum g_\zeta$. As the intuitive picture of addition between two ordinals is just to put one well order after the other one so as to together form a new well order, we know that the ordinal of well order relation $<$ over the set $\{(\beta,\gamma):\beta\in\alpha,\gamma\in f_\alpha(\beta)\}$ is $(\sum g_\zeta)+f_\alpha(\zeta)=(\sum\limits_\zeta f_\alpha)+f_\alpha(\zeta)=\sum\limits_\alpha f_\alpha=\sum f_\alpha$.
    \item If $\alpha$ is a limit ordinal, for any ordinal $\zeta<\alpha$, let $g_\zeta:\zeta\to\Ord=\{(\beta,f_\alpha(\beta)):\beta\in\zeta\}\subset f_\alpha$ be another function, and by induction hypothesis we know that the ordinal of well order relation $<$ over the set $\{(\beta,\gamma):\beta\in\zeta,\gamma\in g_\zeta(\beta)\}$ is $\sum g_\zeta$. As $\alpha$ is a limit ordinal, we have the set $\{(\beta,\gamma):\beta\in\alpha,\gamma\in f_\alpha(\beta)\}=\bigcup\limits_{\zeta<\alpha}\{(\beta,\gamma):\beta\in\zeta,\gamma\in f_\alpha(\beta)\}=\bigcup\limits_{\zeta<\alpha}\{(\beta,\gamma):\beta\in\zeta,\gamma\in g_\zeta(\beta)\}$, so the ordinal of well order relation $<$ over this set is $\bigcup\limits_{\zeta<\alpha}(\sum g_\zeta)=\bigcup\limits_{\zeta<\alpha}(\sum\limits_\zeta f_\alpha)=\sum\limits_\alpha f_\alpha=\sum f_\alpha$.
\end{itemize}
\end{proof}
\begin{lemma}\label{lm.reg}
Let $\aleph_\alpha$ be a fixed infinite regular cardinal. Let $\beta$ be any fixed ordinal such that $\beta\leqslant\aleph_\alpha$, and $f_\beta:\beta\to\aleph_\alpha$ be a fixed function. Then $\sum f_\beta\leqslant\aleph_\alpha$.
\end{lemma}
\begin{proof}
It suffices to prove this statement only when $\beta$ is a limit ordinal. To a contradiction suppose $\sum f_\beta=\sum\limits_\beta f_\beta=\bigcup\limits_{\gamma<\beta}(\sum\limits_\gamma f_\beta)>\aleph_\alpha$, since ordinals are well ordered, there exists the minimal ordinal $\gamma_0<\beta\leqslant\aleph_\alpha$ such that $\sum\limits_{\gamma_0}f_\beta\geqslant\aleph_\alpha$. Obviously $\gamma_0\neq0$. If $\gamma_0$ is a successor ordinal, namely there exists an ordinal $\zeta$ such that $\gamma_0=\zeta+1$, then $\sum\limits_\zeta f_\beta<\aleph_\alpha$ and $f_\beta(\zeta)<\aleph_\alpha$ but $\sum\limits_{\gamma_0}f_\beta=(\sum\limits_\zeta f_\beta)+f_\beta(\zeta)\geqslant\aleph_\alpha$, a contradiction. If $\gamma_0$ is a limit ordinal, then $\gamma_0<\aleph_\alpha$, and $\forall\zeta<\gamma_0$, $\sum\limits_\zeta f_\beta<\aleph_\alpha$, but $\sum\limits_{\gamma_0}f_\beta=\bigcup\limits_{\zeta<\gamma_0}(\sum\limits_\zeta f_\beta)\geqslant\aleph_\alpha$, contradicting that $\aleph_\alpha$ is a regular cardinal. Therefore $\sum f_\beta\leqslant\aleph_\alpha$.
\end{proof}
\begin{theorem}\label{th.squ}
Let $\aleph_\alpha$ be a fixed infinite regular cardinal. Let $\beta$ be any fixed ordinal such that $\beta\leqslant\aleph_\alpha$, and $f_\beta:\beta\to(\aleph_\alpha+1)$ be a fixed function. Then $\sum f_\beta=\aleph_\alpha^2$ iff $\mid\{\gamma\in\beta:f_\beta(\gamma)=\aleph_\alpha\}\mid=\aleph_\alpha$.
\end{theorem}
\begin{proof}
Let $g_\beta:\beta\to(\aleph_\alpha+1)$ be another function such that $\forall\gamma\in\beta$, $g_\beta(\gamma)=\aleph_\alpha\geqslant f_\beta(\gamma)$. Then obviously $\sum f_\beta\leqslant\sum g_\beta=\aleph_\alpha\cdot\beta\leqslant\aleph_\alpha^2$.

For the direction from right to left. As $\{\gamma\in\beta:f_\beta(\gamma)=\aleph_\alpha\}\subseteq\beta$, we have $\mid\beta\mid\geqslant\mid\{\gamma\in\beta:f_\beta(\gamma)=\aleph_\alpha\}\mid=\aleph_\alpha$, and because $\beta\leqslant\aleph_\alpha$, $\beta=\aleph_\alpha$. Relation $<$ over the subset $\{\gamma\in\beta:f_\beta(\gamma)=\aleph_\alpha\}\subseteq\beta$ is also a well order with ordinal less than or equal to $\beta$, namely $\aleph_\alpha$, and because $\mid\{\gamma\in\beta:f_\beta(\gamma)=\aleph_\alpha\}\mid=\aleph_\alpha$, the ordinal of well order relation $<$ over this set is $\aleph_\alpha$, too. Thus consider the subset $\{(\gamma,\zeta):\gamma\in\beta,f_\beta(\gamma)=\aleph_\alpha,\zeta\in f_\beta(\gamma)\}\subseteq\{(\gamma,\zeta):\gamma\in\beta,\zeta\in f_\beta(\gamma)\}$, over which relation $<$ is also a well order with ordinal $\aleph_\alpha^2$, hence by Lemma~\ref{lm.add}, $\sum f_\beta$ is the ordinal of well order relation $<$ over the whole set $\{(\gamma,\zeta):\gamma\in\beta,\zeta\in f_\beta(\gamma)\}$ and so $\sum f_\beta\geqslant\aleph_\alpha^2$. Therefore $\sum f_\beta=\aleph_\alpha^2$.

For the direction from left to right. To a contradiction suppose $\mid\{\gamma\in\beta:f_\beta(\gamma)=\aleph_\alpha\}\mid<\aleph_\alpha$, since $\aleph_\alpha$ is a regular cardinal, there exists an upper bound ordinal $\eta\leqslant\beta\land\eta<\aleph_\alpha$ such that $\forall\eta\leqslant\gamma<\beta$, $f_\beta(\gamma)<\aleph_\alpha$. By Lemma~\ref{lm.add}, $\sum f_\beta$ is the ordinal of well order relation $<$ over the set $\{(\gamma,\zeta):\gamma\in\beta,\zeta\in f_\beta(\gamma)\}=\{(\gamma,\zeta):\gamma<\eta,\zeta\in f_\beta(\gamma)\}\cup\{(\gamma,\zeta):\eta\leqslant\gamma<\beta,\zeta\in f_\beta(\gamma)\}$, then by Lemma~\ref{lm.reg} and the intuitive picture of addition between two ordinals, $\sum f_\beta\leqslant\aleph_\alpha\cdot\eta+\aleph_\alpha=\aleph_\alpha\cdot(\eta+1)<\aleph_\alpha^2$, a contradiction. Therefore $\mid\{\gamma\in\beta:f_\beta(\gamma)=\aleph_\alpha\}\mid=\aleph_\alpha$.
\end{proof}
\begin{remark}
As a concluding remark of this section, the intuitive meaning of Theorem~\ref{th.squ} says that, let $\aleph_\alpha$ be a fixed infinite regular cardinal, then in order to reach $\aleph_\alpha^2$ in total sum by transfinitely adding no more than $\aleph_\alpha$ many ordinals all of which are no greater than $\aleph_\alpha$, on the one hand apparently, just adding together $\aleph_\alpha$ many $\aleph_\alpha$s will work; on the other hand, this is actually an indispensable requirement, namely, there must exist $\aleph_\alpha$ many $\aleph_\alpha$s among all the addenda. Recall Example~\ref{ex.flaw}, readers may speculate that Theorem~\ref{th.squ} will play the crucial role of converting between statements about cardinal $\aleph_\alpha$ and those about ordinal $\aleph_\alpha^2$, while the former is absolute regardless of well orders and the latter can be testified via ordinal logarithm by Theorem~\ref{th.log}. These mathematical results developed within this section will soon get utilized in the following Section~\ref{sec.dtml} so as to strictly define the semantics of modality $\boxbox$ for duplex transfinite modal logic.
\end{remark}

\section{Duplex Transfinite Modal Logic}\label{sec.dtml}

\begin{assumption}[Regular Cardinality]\label{ass.reg}
In the rest of this paper, we additionally assume that all the infinite cardinals which we talk about are regular.
\end{assumption}
\begin{remark}\label{rem.reg}
Readers of course immediately notice that, the above Assumption~\ref{ass.reg} so obviously aims at accommodating to Theorem~\ref{th.squ}, that at first glance it probably seems abruptly artificial. Nevertheless, a while of calm analysis will convince readers that such an assumption is actually reasonable and solid. In fact, simply because the smallest irregular cardinal is $\aleph_\omega$, even restricting our consideration about infinite cardinals only within this set $\{\aleph_n:n\in\omega\}$ still provides us with countably many successive infinite cardinals, which are already sufficient for any practical usage. From another perspective, for now readers are suggested simply to put down this concern anyway, because later on in Section~\ref{sec.fmp}, we shall be devoted to mathematical work for showing Theorem~\ref{th.fmp} together with its following Corollary~\ref{cor.fmp}, which is a counterpart of regular finite model property theorem and actually provides a pretty convincing explanation just for this Assumption~\ref{ass.reg}.
\end{remark}
\begin{definition}[Language $\DTML$]
Let $\pl$ be a non-empty set of propositions. Language $\DTML$ is recursively defined as the following Backus-Naur form, where $p\in\pl$:
\begin{align*}
    \phi::=p\mid\neg\phi\mid(\phi\land\phi)\mid\Box\phi\mid\boxbox\phi
\end{align*}
\end{definition}
Hence, $\STML$ is a sublanguage of $\DTML$. Kripke models keep exactly the same as originally defined in Definition~\ref{def.mod}. Therefore, we only have to supplement semantics definition of the new modality $\boxbox$.
\begin{definition}[Live Successor Cardinality of Pointed Model]\label{def.lscar}
For any pointed model $\M,s$, we define its live successor cardinality as $\lambda_s=\mid\{t\in S:sRt,\exists u\in S,tRu\}\mid$.
\end{definition}
\begin{remark}
Intuitively, ``live successor cardinality of pointed model $\M,s$'' in Definition~\ref{def.lscar} in just the cardinality of live successors of $s$, i.e., successors which are not dead ends and possess some successor on their own. In contrast with $\kappa_s$ in Definition~\ref{def.scar} and Remark~\ref{rem.scar}, here as for $\lambda_s$, we must reasonably ignore those dummy successors which are no longer succeeded by any of their own successors, because they have already been ruled out as \textit{absolutely} impossible in our posterior probability distribution and thus will never be allocated any support of evidence, just like explained in Example~\ref{ex.inf}.
\end{remark}
\begin{definition}[Semantics of $\DTML$]\label{def.dsem}
Semantics of $\DTML$ extends semantics of $\STML$ in Definition~\ref{def.ssem} with the new modality $\boxbox$. First, for any pointed model $\M,s$, arbitrarily fix a bijective function $f_s:\lambda_s\to\{t\in S:sRt,\exists u\in S,tRu\}$. Next, for each possible world $t\in\{t\in S:sRt,\exists u\in S,tRu\}$, arbitrarily fix a bijective function $g_t:\kappa_t\to\{u\in S:tRu\}$. Define the following functions:
\begin{itemize}
    \item $h:\lambda_s\to\Ord$, $\forall\gamma\in\lambda_s$, $h(\gamma)=\kappa_{f_s(\gamma)}$.
    \item $h^+:\lambda_s\to\Ord$, $\forall\gamma\in\lambda_s$, $h^+(\gamma)=$ the ordinal of well order relation under $g_{f_s(\gamma)}$ over the subset $\{u\in S:f_s(\gamma)Ru\land(\M,u\vDash\phi)\}\subseteq\{u\in S:f_s(\gamma)Ru\}$.
    \item $h^-:\lambda_s\to\Ord$, $\forall\gamma\in\lambda_s$, $h^-(\gamma)=$ the ordinal of well order relation under $g_{f_s(\gamma)}$ over the subset $\{u\in S:f_s(\gamma)Ru\land\neg(\M,u\vDash\phi)\}\subseteq\{u\in S:f_s(\gamma)Ru\}$.
\end{itemize}
The semantics of modality $\boxbox$ in language $\DTML$ is then defined as the following:
\begin{align*}
    \M,s\vDash\boxbox\phi\iff & \left\{\begin{array}{llll}
        {[\sum h^+]}_\zeta>[\sum h^-]_\zeta, & & & \text{if }\mid\sum h\mid=\aleph_\zeta\in\InfCard\\
        \\
        \forall t,u\in S\text{ s.t. }sRt\land tRu,\M,u\vDash\phi, & & & \text{otherwise}
    \end{array}\right.
\end{align*}
\end{definition}
In Definition~\ref{def.dsem}, functions $f_s$ and $g_t$ are allowed to be chosen arbitrarily, hence we have to prove that Definition~\ref{def.dsem} is indeed well defined, independent of choices on $f_s$ and $g_t$. But first let us briefly explain the intuition. Complex as its form might seem, Definition~\ref{def.dsem} essentially manages to distinguish between finer-grained ordinals within the classical Kripke model, where no information about any kind of well order among successors of $s$ exists at all and thus only coarser-grained cardinality is absolute. The serious philosophical flaw mentioned in Example~\ref{ex.flaw} can then get favourably dissolved under our logic's currently more advanced resolution, which is able to achieve with the help of Theorem~\ref{th.log} and Theorem~\ref{th.squ}, as readers will surely come to realize the whole vital mechanism through rigorous proof for the following Proposition~\ref{pro.dsem}:
\begin{proposition}\label{pro.dsem}
The semantics of modality $\boxbox$ is well defined in Definition~\ref{def.dsem}, i.e., it does not matter at all how we arbitrarily fix bijective functions $f_s$ and $g_t$.
\end{proposition}
\begin{proof}
We only need to pay attention to the nontrivial case when $\mid\sum h\mid=\aleph_\zeta$ and $\sum h^+,\sum h^-\geqslant1$. As $\kappa_t\geqslant1$ for all $t\in\{t\in S:sRt,\exists u\in S,tRu\}$, we obviously have $\lambda_s,\kappa_t\leqslant\aleph_\zeta$, thus $\sum h,\sum h^+,\sum h^-\leqslant\aleph_\zeta^2$, so $\log_{\aleph_\zeta}(\sum h^+),\log_{\aleph_\zeta}(\sum h^-)\leqslant 2$. By Theorem~\ref{th.log}, $[\sum h^+]_\zeta>[\sum h^-]_\zeta$ iff $\log_{\aleph_\zeta}(\sum h^+)>\log_{\aleph_\zeta}(\sum h^-)$. If $\log_{\aleph_\zeta}(\sum h^+)=1\land\log_{\aleph_\zeta}(\sum h^-)=0$, then $\mid\sum h^+\mid=\aleph_\zeta>\mid\sum h^-\mid$, and cardinality will not vary with different functions $f_s$ and $g_t$. If $\log_{\aleph_\zeta}(\sum h^+)=2\land\log_{\aleph_\zeta}(\sum h^-)<2$, then by Theorem~\ref{th.squ}, $\mid\{\gamma\in\lambda_s:h^+(\gamma)=\aleph_\zeta\}\mid=\aleph_\zeta\land\mid\{\gamma\in\lambda_s:h^-(\gamma)=\aleph_\zeta\}\mid<\aleph_\zeta$. Similarly, cardinality will not vary with different functions $f_s$ and $g_t$. Therefore, the semantics of modality $\boxbox$ in Definition~\ref{def.dsem} is well defined.
\end{proof}
Now from Proposition~\ref{pro.dsem} we know that $\boxbox$ is indeed a well defined modality, but after all is it really a new modality? In fact in Example~\ref{ex.inf}, we have analyzed the close similarity on intuitions of modality $\boxbox$ and two consecutive modalities $\Box\Box$, nevertheless, the following Proposition~\ref{pro.dif} clearly claims their difference:
\begin{proposition}\label{pro.dif}
Generally speaking, neither of the following two formulae is valid:
\begin{align*}
    \boxbox\phi\to\Box\Box\phi\\
    \Box\Box\phi\to\boxbox\phi
\end{align*}
\end{proposition}
\begin{proof}
Please cf. the following Example~\ref{ex.2to11} and Example~\ref{ex.11to2}.
\end{proof}
\begin{example}\label{ex.2to11}
In this Kripke model $\M_{\ref{ex.2to11}}$, possible successor $t_1$ gets a posteriori supported by $\aleph_0$ many new evidences, on all of which $p$ holds. Now we have $\M_{\ref{ex.2to11}},s\vDash\boxbox p\land\neg\Box\Box p$.
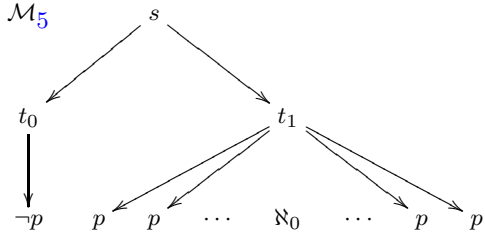
$$\xymatrix@R=25pt@C=10pt{
\M_{\ref{ex.2to11}} & & s\ar@{->}[dll]\ar@{->}[drr] & & & & &\\
t_0\ar@{->}[d] & & & & t_1\ar@{->}[dlll]\ar@{->}[dll]\ar@{->}[drr]\ar@{->}[drrr] & & &\\
\neg p & p & p & \cdots & \aleph_0 & \cdots & p & p
}$$
\end{example}
\begin{example}\label{ex.11to2}
This Kripke model $\M_{\ref{ex.11to2}}$ is a reexamination of the Kripke model $\M_{\ref{ex.inf}}$ in Example~\ref{ex.inf}. Now we have $\M_{\ref{ex.11to2}},s\vDash\neg\boxbox p\land\Box\Box p$.
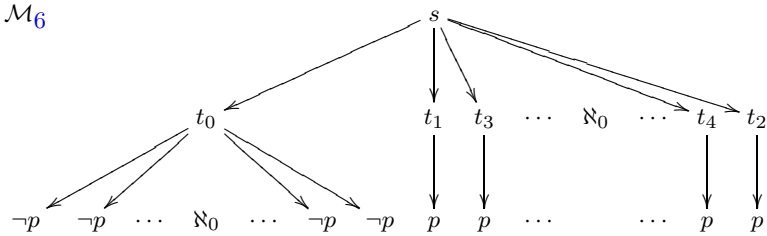
$$\xymatrix@R=25pt@C=5pt{
\M_{\ref{ex.11to2}} & & & & & & & s\ar@{->}[dllll]\ar@{->}[d]\ar@{->}[dr]\ar@{->}[drrrrr]\ar@{->}[drrrrrr] & & &\\
& & & t_0\ar@{->}[dlll]\ar@{->}[dll]\ar@{->}[drr]\ar@{->}[drrr] & & & & t_1\ar@{->}[d] & t_3\ar@{->}[d] & \cdots & \aleph_0 & \cdots & t_4\ar@{->}[d] & t_2\ar@{->}[d]\\
\neg p & \neg p & \cdots & \aleph_0 & \cdots & \neg p & \neg p & p & p & \cdots & & \cdots & p & p
}$$
\end{example}
At last, we conclude this section by showing that this new modality $\boxbox$ successfully solves our target problem proposed in Example~\ref{ex.flaw}, as the following Example~\ref{ex.fix} illustrates:
\begin{example}\label{ex.fix}
This Kripke model $\M_{\ref{ex.fix}}$ is a reexamination of the Kripke model $\M_{\ref{ex.flaw}}$ in Example~\ref{ex.flaw}. Now we have $\M_{\ref{ex.fix}},s\vDash\boxbox p\land\Box\Box p$, therefore, the flaw in Example~\ref{ex.flaw} has gotten fixed satisfactorily.
$$
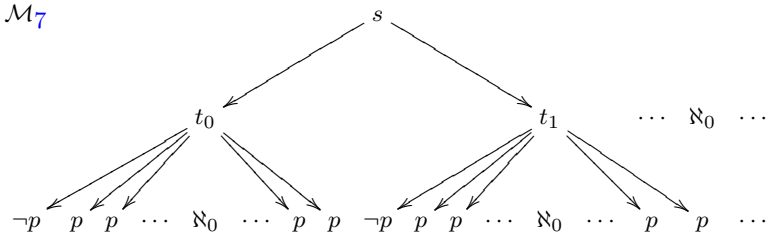
\xymatrix@R=25pt@C=2.5pt{
\M_{\ref{ex.fix}} & & & & & & & & s\ar@{->}[dllll]\ar@{->}[drrrr] & & & & & & & &\\
& & & & t_0\ar@{->}[dllll]\ar@{->}[dlll]\ar@{->}[dll]\ar@{->}[drr]\ar@{->}[drrr] & & & & & & & & t_1\ar@{->}[dllll]\ar@{->}[dlll]\ar@{->}[dll]\ar@{->}[drr]\ar@{->}[drrr] & & \cdots & \aleph_0 & \cdots\\
\neg p & p & p & \cdots & \aleph_0 & \cdots & p & p & \neg p & p & p & \cdots & \aleph_0 & \cdots & p & p & \cdots
}$$
\end{example}

\section{Sherlock Holmes' Saying}\label{sec.det}

We are now ready for an exciting application of transfinite modal logic as explanation for Bayesian Reasoning: a perfect formalization of Sherlock Holmes' famous saying, ``When you have eliminated the impossible, whatever remains, however improbable, must be the truth.'' Intuitively, in this scenario a radical conversion happens: we drastically revise our belief from originally thinking something improbable to finally assuring it as the only necessary truth. Such intensive vibration does not seem usual in logic, but it does take place so naturally in our daily life even if we are not a detective, for instance as we have demonstrated, winning the first prize of a lottery. We shall see that in transfinite modal logic, even stronger evidence is required in order for the improbable as of the prior probability to become the necessary as of the posterior probability, which definitely fits pretty well with our intuition about how we overthrow an existing belief and establish the opposite one.
\begin{example}\label{ex.det}
In this Kripke model $\M_{\ref{ex.det}}$, where $i,j\in\omega$, for the prior probability distribution there are $\aleph_i$ many supports for $\neg p$ but only one support for $p$, hence by default $p$ is improbable, namely $\M_{\ref{ex.det}},s\vDash\Box\neg p$. However at the next time point, $\aleph_j$ many new evidences for $p$ are discovered while each possible $\neg p$ branch is only assigned with one evidence, thus now it settles on the comparison of magnitude between $i$ and $j$.

For one case suppose $i<j$, namely the new evidence is even stronger so as to deny our a priori assumption, and accordingly by our Bayesian reasoning, we a posteriori have $\M_{\ref{ex.det}},s\vDash\boxbox p$ in contrast.
$$
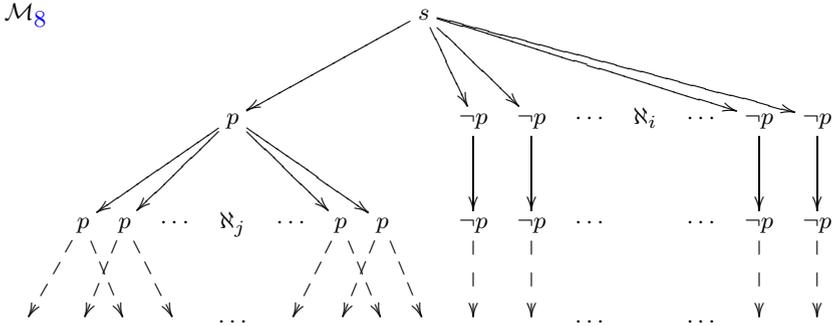
\xymatrix@R=25pt@C=5pt{
\M_{\ref{ex.det}} & & & & & & & & s\ar@{->}[dllll]\ar@{->}[dr]\ar@{->}[drr]\ar@{->}[drrrrrr]\ar@{->}[drrrrrrr] & & & &\\
& & & & p\ar@{->}[dlll]\ar@{->}[dll]\ar@{->}[drr]\ar@{->}[drrr] & & & & & \neg p\ar@{->}[d] & \neg p\ar@{->}[d] & \cdots & \aleph_i & \cdots & \neg p\ar@{->}[d] & \neg p\ar@{->}[d]\\
& p\ar@{-->}[dl]\ar@{-->}[dr] & p\ar@{-->}[dl]\ar@{-->}[dr] & \cdots & \aleph_j & \cdots & p\ar@{-->}[dl]\ar@{-->}[dr] & p\ar@{-->}[dl]\ar@{-->}[dr] & & \neg p\ar@{-->}[d] & \neg p\ar@{-->}[d] & \cdots & & \cdots & \neg p\ar@{-->}[d] & \neg p\ar@{-->}[d]\\
& & & & \cdots & & & & & & & \cdots & & \cdots & &
}$$
For another case suppose $i=j$, namely the new evidence is approximately of the same weight as our default bias, then it is able to shock our original belief but yet not strong enough to establish its own, hence we have $\M_{\ref{ex.det}},s\vDash\neg\boxbox\neg p\land\neg\boxbox p$. Nonetheless, as we indicate with dotted line in Kripke model $\M_{\ref{ex.det}}$, if later on we continue obtaining more and more new evidences supporting for $p$, we shall finally turn our belief completely over. As mentioned in Example~\ref{ex.inf}, although in this paper we do not formally introduce more generalized modalities such as $\Box^3$ or $\Box^4$, readers should be able to reasonably imagine that in such a condition we may end with $\M_{\ref{ex.det}},s\vDash\Box^3p$. Of course on the other hand, depending on the actual situation, it may also turn out that we later on obtain more new evidences supporting for $\neg p$ and thus result in $\M_{\ref{ex.det}},s\vDash\Box^3\neg p$.

For the final case suppose $i>j$, namely the new evidence is still too weak to currently impose any visible impact, and so we keep holding $\M_{\ref{ex.det}},s\vDash\boxbox\neg p$. Nevertheless, those $\aleph_j$ many new evidences are never utterly useless: we may also imagine that if in the next dotted round, so many possibilities for $\neg p$ even become dead end that the total amount of evidences for $\neg p$ decreases to $\aleph_j$ or even fewer, then these $\aleph_j$ many new evidences for $p$ will start to play their own role regarding $\Box^3$ formulae.
\end{example}
As a conclusion, we have fully learnt from the above Example~\ref{ex.det} that it is the relative magnitude between cardinals or ordinals, rather than their absolute value, that ultimately matters to transfinite modal logic's formulae. Different relative magnitudes stand for different levels of strength of evidence and can be quantitatively calculated and compared by modal formulae. Therefore, although as mentioned in Example~\ref{ex.fin}, transfinite modal logic is not expressible enough for arbitrary specific value of probability between $0$ and $1$, it is after all expressible to some extent for a certain degree of quantitative computation on probability and thus we view it as a semi-quantitative explanation for Bayesian reasoning.

\section{``Finite'' Model Property}\label{sec.fmp}

We have already largely justified our transfinite modal logic's soundness as well as usefulness, through both conceptual analysis and specific examples. Nevertheless just in case, extremely careful readers might still wonder that our Assumption~\ref{ass.reg} looks more or less artificial anyway. Now in this section, we offer a reconsideration of such an issue, and decide to settle on solving this problem permanently via strict mathematical argument, as we have previously forecast in Remark~\ref{rem.reg}. In detail, we shall show that transfinite modal logic enjoys a counterpart of finite model property: for any $\DTML$ formula $\phi$, if $\phi$ is satisfiable in some Kripke model, then $\phi$ is also satisfiable in a Kripke model with fewer than $\aleph_\omega$ many possible worlds. Of course, since here infinity plays a crucial role in our logic, we do not expect that $\phi$ can always be satisfied in some \textit{really} finite Kripke model; however, we shall prove that $\phi$ can always be satisfied in an $\aleph_n$-large Kripke model, where $n\in\omega$, and as explained in Remark~\ref{rem.fin}, this may be understood simply as a counterpart of ordinary finite model property theorem~\cite{Libkin04}. Hence all in all, readers should be anyhow convinced that Assumption~\ref{ass.reg} is actually reasonable, innocent and self-consistent.
\begin{definition}[Degree of Formula]
For any $\DTML$ formula $\phi$, its degree $\deg(\phi)\in\omega$ is recursively defined as the following:
\begin{align*}
    \deg(p)= & 0\\
    \deg(\neg\phi)= & \deg(\phi)\\
    \deg(\phi\land\psi)= & \max\{\deg(\phi),\deg(\psi)\}\\
    \deg(\Box\phi)= & \deg(\phi)+1\\
    \deg(\boxbox\phi)= & \deg(\phi)+2
\end{align*}
\end{definition}
\begin{definition}[Walk in Kripke Model]
Let $\M,s$ be a fixed pointed model. A walk starting from $s$ with length $n\in\omega$ is a function $w_n:n\to S$ such that $sRw_n(0)$ if $n>0$, and that for any $k\in\omega\land0\leqslant k<n-1$, $w_n(k)Rw_n(k+1)$.
\end{definition}
\begin{definition}[Unravelled Pointed Model]
Let $\M,s$ be a fixed pointed model. For any $k\in\omega$, the unravelled pointed model until depth $k$ is a pointed model $\M_k,\emptyset$, where $\M_k=(S_k,R_k,V_k)$ and:
\begin{itemize}
    \item $S_k=\{w_n:w_n$ is a walk starting from $s$ with length $n\in\omega,n\leqslant k\}$. $w_0=\emptyset\in S_k$.
    \item For any $w_m,w_n\in S_k$, where $m,n\in\omega\land m,n\leqslant k$, $w_mR_kw_n$ iff $m+1=n\land w_m\subset w_n$.
    \item For any $p\in\pl$ and any $w_n\in S_k$, where $n\in\omega\land n\leqslant k$, if $n=0$, then $w_n\in V_k(p)$ iff $s\in V(p)$; otherwise, $w_n\in V_k(p)$ iff $w_n(n-1)\in V(p)$.
\end{itemize}
\end{definition}
\begin{proposition}\label{pro.unr}
Let $\M,s$ be a fixed pointed model. For any $k\in\omega$, the unravelled pointed model $\M_k,\emptyset$ is a tree with root $\emptyset$ and maximal depth $k$, moreover, for any $\DTML$ formula $\phi$ such that $\deg(\phi)\leqslant k$, $\M,s\vDash\phi\iff\M_k,\emptyset\vDash\phi$.
\end{proposition}
\begin{proof}
Obvious.
\end{proof}
Therefore by Proposition~\ref{pro.unr}, $\phi$ is satisfiable in some Kripke model if and only if $\phi$ is satisfiable in some tree-like Kripke model with maximal depth $\deg(\phi)$. Till now, our reasoning is no different from usual standard proof of normal modal logic's finite model property~\cite{Blackburn01}. However, presently this finite-depth tree-like Kripke model may still bear arbitrarily large width, hence requiring our further treatment.
\begin{proposition}\label{pro.fin}
If $\pl$ is finite, then for any $k\in\omega$, there exist only finite many nonequivalent $\DTML$ formulae $\phi$ such that $\deg(\phi)\leqslant k$.
\end{proposition}
\begin{proof}
By induction on $k$, obvious.
\end{proof}
\begin{definition}[Modal Equivalence]\label{def.meq}
For any $k\in\omega$ and any $\alpha,\beta\in\Ord$, any two pointed models $\M,s$ and $\N,w$ are $(k,\alpha,\beta)$-modal-equivalent iff:
\begin{itemize}
    \item For any $\DTML$ formula $\phi$ such that $\deg(\phi)\leqslant k$, $\M,s\vDash\phi\iff\N,w\vDash\phi$.
    \item For any $\DTML$ formula $\phi$ such that $\deg(\phi)<k$, either one of the following two cases holds:
    \begin{enumerate}
        \item $\neg\exists t$, $t$ is a successor of $s$ and $\M,t\vDash\phi$, $\neg\exists v$, $v$ is a successor of $w$ and $\N,v\vDash\phi$.
        \item $\exists t$, $t$ is a successor of $s$ and $\M,t\vDash\phi$, $\exists v$, $v$ is a successor of $w$ and $\N,v\vDash\phi$, furthermore, $\log_{\aleph_\alpha}\mid\{t$ is a successor of $s:\M,t\vDash\phi\}\mid=\log_{\aleph_\beta}\mid\{v$ is a successor of $w:\N,v\vDash\phi\}\mid$.
    \end{enumerate}
\end{itemize}
\end{definition}
\begin{lemma}\label{lm.bla}
For any tree-like pointed models $\M,s$ and $\N,w$ and any fixed $k\in\omega$, suppose $\mid\{u:u$ is a successor of a successor of $s\}\mid=\aleph_\alpha\in\InfCard$, $\mid\{r:r$ is a successor of a successor of $w\}\mid=\aleph_\beta\in\InfCard$, where $\alpha,\beta\in\Ord$, for any successor $t_0$ of $s$, $\log_{\aleph_\alpha}\mid\{t$ is a successor of $s:\M,t_0$ and $\M,t$ are $(k+1,\alpha,\alpha)$-modal-equivalent$\}\mid=\log_{\aleph_\beta}\mid\{v$ is a successor of $w:\M,t_0$ and $\N,v$ are $(k+1,\alpha,\beta)$-modal-equivalent$\}\mid$, and for any successor $v_0$ of $w$, vice versa. Then for any $\DTML$ formula $\phi$ such that $\deg(\phi)\leqslant k$, $\M,s\vDash\boxbox\phi\iff\N,w\vDash\boxbox\phi$.
\end{lemma}
\begin{proof}
This result is not difficult to observe from Definition~\ref{def.dsem}.
\end{proof}
\begin{theorem}\label{th.fmp}
If $\pl$ is finite, then for any tree-like pointed model $\M,s$, any $k\in\omega$ and any $\alpha\in\Ord$, if $\mid\{t\in S:sRt\}\mid\leqslant\aleph_\alpha$, then there exists $m\in\omega$ so that for any $n\in\omega\land n\geqslant m$, there exists a tree-like pointed model $\N_n,w$ such that $\M,s$ and $\N_n,w$ are $(k,\alpha,n)$-modal-equivalent and that size of Kripke model $\N_n$ is smaller than $\aleph_\omega$.
\end{theorem}
\begin{proof}
By induction on $k$. When $k=0$, apparently keeping the single possible world $s$ is enough for modal equivalence, so we focus on the inductive step when $k>0$.

For one case, if $\mid\sum h\mid=\mid\{u\in S:\exists t\in S,sRt\land tRu\}\mid=\aleph_\beta\in\InfCard$, where $\beta\in\Ord$ and $h$ is defined as in Definition~\ref{def.dsem}, then for any $t\in\{t\in S:sRt,\exists u\in S,tRu\}$, we have $\{u\in S:tRu\}\subseteq\{u\in S:\exists t\in S,sRt\land tRu\}$ so $\mid\{u\in S:tRu\}\mid\leqslant\aleph_\beta$ and thus $\log_{\aleph_\beta}\mid\{u\in S:tRu\}\mid\leqslant1$. Therefore by Proposition~\ref{pro.fin}, there exist only finite many $(k-1,\beta,\beta)$-modal-equivalence classes over the set $\{\M,t:t\in S,sRt\}$. For every such equivalence class, by induction hypothesis, there exists $m\in\omega$ with respect to arbitrary $\M,t$ in this equivalence class, $(k-1)\in\omega$ and $\beta\in\Ord$. Since there are only finite many equivalence classes, we can take the largest $m_0\in\omega$ and so for any $l\in\omega\land l\geqslant m_0$, any equivalence class and arbitrary $\M,t$ in the class, there exists a tree-like pointed model $\L_l,v$ such that $\M,t$ and $\L_l,v$ are $(k-1,\beta,l)$-modal-equivalent and that size of Kripke model $\L_l$ is smaller than $\aleph_\omega$.

Now we claim that $(m_0+3)\in\omega$ just satisfies this theorem regardless of $\alpha$, hence for any fixed $\alpha\in\Ord$ such that $\aleph_\alpha\geqslant\mid\{t\in S:sRt\}\mid$ and any fixed $n\in\omega\land n\geqslant m_0+3$, we have to construct our target tree-like pointed model $\N_n,w$. To start with, there is only one possible world $w$ as the tree root in Kripke model $\N_n$, with the same propositional valuation as $s$. For each equivalence class, since $\emptyset\subset\{t\in S:\M,t$ is in the equivalence class$\}\subseteq\{t\in S:sRt\}$, cardinality of this equivalence class taking logarithm with respect to base $\aleph_\alpha$ will only result in either $0$ or $1$. Moreover when $k>1$, it is easy to see that an equivalence class contains either all live possible worlds (i.e. with some successor) or all dead possible worlds (i.e. without any successor), and in the former case since $\mid\{t\in S:sRt,\exists u\in S,tRu\}\mid\leqslant\mid\{u\in S:\exists t\in S,sRt\land tRu\}\mid=\aleph_\beta$, cardinality of this equivalence class taking logarithm with respect to base $\aleph_\beta$ will only result in either $0$ or $1$, too. Let $l=n+1$ if $\alpha<\beta$, $l=n$ if $\alpha=\beta$, $l=n-2$ if $\alpha>\beta$, in any case, $l\in\omega\land l\geqslant m_0+1$. We then modify Kripke model $\N_n$ through the following sequent procedure:
\begin{enumerate}
    \item Firstly, for each equivalence class whose logarithm with respect to $\aleph_\alpha$ is $1$, we add $\aleph_n$ many independent copies of the corresponding pointed model $\L_l,v$ into Kripke model $\N_n$ as the children of $w$.
    \item Secondly, if $k>1$, then for each remaining equivalence class such that it contains live possible worlds and that its logarithm with respect to $\aleph_\beta$ is $1$, we add $\aleph_l$ many independent copies of the corresponding pointed model $\L_l,v$ into Kripke model $\N_n$ as the children of $w$. Temporarily take $c\in\omega$ just for recording an ad hoc variable, if some equivalence class gets processed within this step, we either let $c=l$ if this equivalence class has the same cardinality as $\{t\in S:sRt\}$, or let $c=l+1$ otherwise; if no equivalence class gets processed within this step, we let $c=m_0$.
    \item Thirdly, if $\{t\in S:sRt\}$ is infinite, then for each remaining equivalence class whose cardinality is the same as $\{t\in S:sRt\}$, we add $\aleph_c$ many independent copies of the corresponding pointed model $\L_l,v$ into Kripke model $\N_n$ as the children of $w$.
    \item Finally, for each remaining equivalence class, we add only $1$ copy of the corresponding pointed model $\L_l,v$ into Kripke model $\N_n$ as the child of $w$.
\end{enumerate}
Obviously in the end, size of Kripke model $\N_n$ still keeps smaller than $\aleph_\omega$. The whole constructive procedure may seem a bit complex, nonetheless with particular carefulness and also using Lemma~\ref{lm.bla}, there actually does not present any principled difficulty in checking that pointed models $\M,s$ and $\N_n,w$ are indeed $(k,\alpha,n)$-modal-equivalent.

For another case, if $\{u\in S:\exists t\in S,sRt\land tRu\}$ is finite, the construction is similar but even simpler, so we omit it here.
\end{proof}
\begin{corollary}\label{cor.fmp}
For any $\DTML$ formula $\phi$, $\phi$ is satisfiable in some Kripke model if and only if $\phi$ is satisfiable in some Kripke model that is smaller than $\aleph_\omega$.
\end{corollary}
\begin{remark}
We conclude this section with a brief remark on the intuition of proof. As explained in Example~\ref{ex.inf}, modality $\boxbox$ is in fact ``non-associative'' from consecutive $\Box\Box$, therefore in Definition~\ref{def.meq}, modal equivalence between a pair of pointed models not only requires total agreement on truth values of formulae as usual, but also demands further constraints so that such equivalence relation is able to get preserved between their predecessors as well. Such definition leads to Theorem~\ref{th.fmp}, where kind of like downward L\"owenheim-Skolem theorem~\cite{Marker02}, we manage to compress the Kripke model into the scope smaller than $\aleph_\omega$. Essentially, since a fixed finite formula has only finite length and finite many propositions in it, its resolution should also be restricted within a finite level, hence it can always be expressed inside the room below $\aleph_\omega$, a space vast enough so as to dwell all the first countably many infinite regular cardinals. In a word, although the Kripke model still has to be infinite most of the time, from our discussion in Remark~\ref{rem.fin} readers can get enlightened that, if we choose to interpret infinite cardinals as very large natural numbers instead of actual infinity, it then sounds definitely plausible just to call Corollary~\ref{cor.fmp} as ``finite'' model property. And therefore, readers could finally feel completely relieved about our Assumption~\ref{ass.reg} with overall harmony.
\end{remark}

\section{Conclusions}\label{sec.con}

In this paper, we introduce transfinite modal logic, which combines modal logic with ordinal arithmetic, and apply it as a semi-quantitative explanation for Bayesian reasoning through bountiful examples, especially including a though formalization of the well-known Sherlock Holmes' saying. The philosophical intuition is rather clear and straightforward while the technical details are quite nontrivial, and through proving a counterpart of finite model property theorem we in turn justifies the philosophical playground of our transfinite modal logic. As directions of future work, more aspects of mathematical properties of this logic are to be considered, and we may also compare it with other existing logics about cardinals and ordinals~\cite{Ding20}.

\section*{Acknowledgments}

The author would like to thank Satoshi Tojo for help on notations, Yang Song for philosophical inspiration, and Hiroakira Ono for suggestions on composition.

\end{document}